\documentclass[11pt]{article}
\usepackage[OT4]{fontenc}
\usepackage[margin=1.2in]{geometry}
\usepackage{amssymb,amsfonts,amsmath,amsthm,amscd,dsfont,mathrsfs,pifont}
\usepackage{blkarray}
\usepackage{graphicx,float,psfrag,epsfig,color}
\usepackage{qtree}
\usepackage{comment}
\usepackage[dvipsnames]{xcolor}
\usepackage{authblk,textcomp,xcolor}

\newcommand{\CPMacro}{\CP(\orb(\overline{f_n}), \cU^N)}

\usepackage[pdftex,pagebackref=true,colorlinks]{hyperref}
\hypersetup{linkcolor=[rgb]{.7,0,0}}
\hypersetup{citecolor=[rgb]{0,.7,0}}
\hypersetup{urlcolor=[rgb]{.7,0,.7}}

\newcommand\independent{\protect\mathpalette{\protect\independent}{\perp}}
\def\independent#1#2{\mathrel{\rlap{$#1#2$}\mkern2mu{#1#2}}}

\theoremstyle{plain}
\newtheorem{definition}{Definition}
\theoremstyle{plain}
\newtheorem{theorem}{Theorem}
\theoremstyle{plain}
\theoremstyle{plain}
\theoremstyle{plain}
\theoremstyle{plain}
\newtheorem{proposition}{Proposition}
\theoremstyle{plain}
\theoremstyle{plain}
\newtheorem{lemma}{Lemma}
\theoremstyle{plain}
\newtheorem{corollary}{Corollary}
\theoremstyle{plain}
\theoremstyle{remark}
\newtheorem{remark}{Remark}
\theoremstyle{remark}
\theoremstyle{plain}

\definecolor{DSgray}{cmyk}{0,0,0,0.7}
\definecolor{DSred}{cmyk}{0,0.7,0,0.7}

\newcommand{\pr}{\mathbb{P}}
\newcommand{\E}{\mathbb{E}}
\DeclareMathOperator{\CP}{CP}

\newcommand{\Prr}{\mathop{\rm Pr}\nolimits}

\DeclareMathOperator{\ReLU}{ReLU}
\DeclareMathOperator{\Rad}{Rad}
\DeclareMathOperator{\INAL}{INAL}
\DeclareMathOperator{\sign}{sgn}
\DeclareMathOperator{\erf}{erf}

\DeclareMathOperator{\orb}{orb}

\DeclareMathOperator{\NN}{NN}
\DeclareMathOperator{\poly}{poly}
\newcommand{\cF}{\mathcal{F}}
\newcommand{\cX}{\mathcal{X}}
\newcommand{\cU}{\mathcal{U}}

\newcommand{\cN}{\mathcal{N}}
\newcommand{\cP}{\mathcal{P}}
\newcommand{\bR}{\mathbb{R}}
\newcommand{\bN}{\mathbb{N}}

\usepackage[capitalize,noabbrev]{cleveref}

\usepackage{dsfont}

\usepackage{hyperref}
\hypersetup{pdftitle={OnlineLearning},%
            colorlinks, linktocpage=true, pdfstartpage=1, pdfstartview=FitV,%
    breaklinks=true, pdfpagemode=UseNone, pageanchor=true, pdfpagemode=UseOutlines,%
    plainpages=false, bookmarksnumbered, bookmarksopen=true, bookmarksopenlevel=1,%
    hypertexnames=true, pdfhighlight=/O,%
    urlcolor=orange, citecolor=RoyalBlue}

\author{Emmanuel Abbe}
\author{Elisabetta Cornacchia}
\author{Jan Hązła}
\author{Christopher Marquis}
\affil{\small Ecole Polytechnique F\'{e}d\'{e}rale de Lausanne (EPFL) \\ {\tt\{first.last\}@epfl.ch}}

\title{An initial alignment between neural network and target \\ is needed for gradient descent to learn}

\begin{document}

\maketitle
\begin{abstract}This paper introduces the
notion of ``Initial Alignment'' (INAL)
between a neural network at initialization and
a target function. It is proved that if a network and a Boolean target function do
not have a noticeable INAL,
then noisy gradient descent on a fully connected network with normalized
i.i.d.\ initialization
will not learn in polynomial time. Thus a certain amount of knowledge about the target (measured by the INAL) is needed in the architecture design. This also provides an answer to an open problem posed in~\cite{AS20}. The results are based on deriving lower-bounds for descent algorithms on symmetric neural networks without explicit knowledge of the target function beyond its INAL.
\end{abstract}

\section{Introduction}
Does one need an educated guess on the type of architecture needed in order for gradient descent to learn certain target functions? Convolutional neural networks (CNNs) have an architecture that is natural for learning functions having to do with image features: at initialization, a CNN is already well posed to pick up correlations with the image content due to its convolutional and pooling layers, and gradient descent (GD) allows to locate and amplify such correlations. However, a CNN may not be the right architecture for non-image based target functions, or even certain image-based functions that are non-classical \cite{failing}.
More generally, we raise the following question:
\begin{quote}
    Is a certain amount of `initial alignment' needed between a neural network at initialization and a target function in order for GD to learn on a reasonable horizon? Or could a neural net that is not properly designed but large enough find its own path to correlate with the target?
\end{quote}
In order to formalize the above question, one needs to define the notion of `alignment' as well as to quantify the `certain amount' and `reasonable horizon' notions. This paper focuses on the `polynomial-scaling' regime and on fully connected architectures,  but we conjecture that a more general quantitative picture can be derived. Before defining the question formally, we stress a few connections to related problems.

A different type of `gradual' question has recently been investigated for neural networks, namely, the `depth gradual correlation' hypothesis. This postulates that if a neural network of low depth (e.g., depth 2) cannot learn to a non-trivial accuracy after GD has converged, then an augmentation of the depth to a larger constant will not help in learning \cite{malach2019deeper, allen2020backward}. In contrast, the question studied here is more of a `time gradual correlation' hypothesis, saying that if at time zero GD cannot correlate non-trivially with a target function (i.e., if the neural net at time zero does not have an initial alignment), then a polynomial number of GD steps will not help.

From a lower-bound point of view, the question we ask is also slightly different than the traditional lower-bound questions posed in the learning literature that have to do with the difficulties of learning a class of functions irrespective of a specific architecture. For instance, it is known from \cite{blum1994weakly, kearns1998efficient} that the larger the statistical dimension of a function class is, the more challenging it is for a statistical query (SQ) algorithm to learn, and similarly for GD-like algorithms~\cite{abbe2021power}; these bounds hold irrespective of the type of neural network architectures used.


A more architecture-dependent lower-bound is derived in~\cite{abbe2020poly}, where the junk-flow is essentially used as replacement of the number of queries, and which depends on the type of architecture and initialization albeit being implicit. In~\cite{shalev2021computational}, a separation between fully connected and CNN architectures is obtained, showing that certain target functions have a locality property and are better learned by the latter architecture. In a different setting,~\cite{tan2021cautionary} gives a generalization lower bound for decision trees on additive generative models, proving that decision trees are statistically inefficient at estimating additive regression functions. However, none of the bounds in these works give an explicit figure of merit to measure the suitability of a neural network architecture for a target.

One can interpret such bounds, especially the one in \cite{abbe2020poly}, as follows. If the function class is such that for two functions $F,F'$ sampled randomly from the class, the typical correlation is not noticeable, i.e., if the cross-predictability (CP) is given by
\begin{align} \label{eq:CP_small}
   \CP(F,F'):= \E_{F,F'}  \langle F,F' \rangle^2=n^{-\omega_n(1)},
\end{align}
(where we denoted by $\langle .\rangle$ the $L^2$-scalar product, namely, for some input distribution $P_\cX$, $\langle f,g\rangle = \E_{x\sim P_\cX} [ f(x) g(x)] $ and by $\omega_n(1) $ any sequence that is diverging to $\infty$ as $n \to \infty$),
then GD with polynomial precision and on a polynomial horizon will not be able to identify the target function with an inverse polynomial accuracy (weak learning), because at no time the algorithm will approach a good approximation of the target function; i.e. the gradients stay essentially agnostic to the target.

Instead, here we focus on a specific function --- rather than a function class --- and on a specific architecture and initialization. One can engineer a function class from a specific function if the initial architecture has some distribution symmetry. In such case,
if the original function is learnable,
then its orbit under the group of symmetry must also be learnable, and thus lower bounds based on the cross-predictability or statistical dimension of the orbit can be used. Such lower bounds are no longer applying to any architecture but  exploit the symmetry of the architecture, however they still require knowledge of the target function in order to define the orbit.

In this paper, we would like to depart from the setting where we know the target function and thus can analyze the orbit
directly.
Instead, we would like to have a `proxy' that depends on the underlying  target function
and the initialized neural net
$\NN_{\Theta^0}$ at hand, where the set of weights at time zero $\Theta^0$ are drawn according to some distribution. In~\cite{AS20}, the following proposal is made (the precise statement will appear below): can we replace the correlation among a function class by the correlation between  a target function and an initialized net in order to have a necessary requirement for learning, i.e., if
\begin{align} \label{eq:INAL_small}
    \E_{\Theta^0}  \langle f,\NN_{\Theta^0} \rangle^2=n^{-\omega_n(1)},
\end{align}
or in other words, if at initialization the neural net correlates negligibly with the target function, is it still possible for GD to learn\footnote{Even with just an inverse polynomial accuracy, a.k.a., weak learning.} the function $f$  if the number of epochs of GD is polynomial? We next formalize the question further and provide an answer to it.

Note the difference between~\eqref{eq:CP_small} and~\eqref{eq:INAL_small}: in~\eqref{eq:CP_small} it is the class of functions that is too poorly correlated for {\it any}  SQ algorithm to efficiently learn; in~\eqref{eq:INAL_small} it is the specific network initialization that is too poorly correlated with the specific target in order for GD to efficiently learn.

While previous works and our proof relies on creating the orbit of a target function using the network symmetries and then arguing from the complexity of the orbit (using cross-predictability~\cite{AS20}), we believe that the INAL approach can be fruitful in additional contexts. In fact, the orbit approaches have two drawbacks: (1) they cannot give lower-bounds on functions like the full parity\footnote{we call full parity the function $f:\{\pm 1\}^n \to \{\pm \}$ s.t. $f(x) = \prod_{i=1}^n x_i$.} that have no complex orbit (in fact the orbit of the full parity is itself under permutation symmetries), (2) to estimate the complexity measure of the orbit class (e.g., the cross-predictability) from a sample set without full access to the target function, one needs labels of data points under the group action that defines the orbit (e.g., permutations), and these may not always be available from an arbitrary sample set. In contrast, (i) the INAL can still be small for the full parity function on certain symmetric neural networks, suggesting that in such cases the full parity is not learnable (we do not prove this here due to our specific proof technique but conjecture that this result still holds), (ii) the INAL can always be estimated from a random i.i.d.\ sample set, using basic Monte Carlo simulations (as used in our experiments, see Section~\ref{exper}).

While the notion of INAL makes sense for any input distribution, our theoretical
results are proved in a more limited setting of Boolean functions with uniform inputs. This follows the approach that has been taken in~\cite{abbe2020poly} and we made that choice for similar reasons. Furthermore, any computer-encoded function is eventually Boolean and major part of the PAC learning theory has indeed focused on Boolean functions (we refer to~\cite{shalev2014uml} for more on this subject). We nonetheless expect that the footprints of the proofs derived in this paper will apply to inputs that are iid Gaussians or spherical, using different basis than the Fourier-Walsh one.

Our general strategy in obtaining such a result is as follows: we first show that for the type of architecture considered, a low initial alignment (INAL) implies that the implicit target function is essentially high-degree in its Fourier basis; this part is specific to the architecture and the low INAL property. We next use the symmetry of the initialization to conclude that learning under such high-degree Fourier requirement implies learning a low CP class, and thus conclude by leveraging the results from \cite{abbe2020poly}. Finally, we do some experiments with the types of architecture used in our formal results, but also with  convolutional neural nets to test the robustness of the original conjecture. We observe that generally the INAL gives a decent proxy for the difficulty to learn (lower INAL gives lower learning accuracy). While this goes beyond the scope of our paper --- which is to obtain a first rigorous validation of the INAL conjecture for standard fully connected neural nets --- we believe that the numerical simulations give some motivations to pursue the study of the INAL in a more general setting.

\section{Definitions and Theoretical Contributions}
\label{sec:definitions}


For the purposes of our definition,
a neural network $\NN$ consists of a set of
neurons $V_{\NN}$, a random variable
$\Theta^0\in\mathbb{R}^k$ which corresponds to the initialization and
a collection of functions $\NN_{\Theta^0}^{(v)}:\mathbb{R}^n\to\mathbb{R}$
indexed with $v\in V_{\NN}$,
representing the outputs of neurons in the network.
The Initial Alignment (INAL) is defined as the average squared correlation between the target
function and any of the neurons at initialization:

\begin{definition}[Initial Alignment (INAL)]
Let $f:\mathbb{R}^n\to\mathbb{R}$ be a function and $P_{\cX}$ a distribution
on $\mathbb{R}^n$. Let $\NN$ be a neural network with neuron set $V_{\NN}$ and random initialization $\Theta^0$.
Then, the $\INAL$ is defined as
\begin{align}
    \INAL(f, \NN) := \max_{v \in V_{\NN} } \E_{\Theta^0 } \langle f,\NN_{\Theta^0}^{(v)} \rangle^2,
\end{align}
where we denoted by $\langle .\rangle$ the $L^2$-scalar product, namely $\langle f,g\rangle = \E_{x\sim P_\cX} [ f(x) g(x)] $.
\end{definition}
While the above definition makes sense for any neural network architecture, in this paper we focus on fully connected networks. Thus, in the following $\NN$ will denote a fully connected neural network.
Our main thesis is that in many settings a small INAL is bad news:
If at initialization there is no noticeable correlation between any of the
neurons and the target function, the GD-trained neural network
will not be able to recover such correlation during training in polynomial
time.

Of particular interest to us is the notion of INAL for a single
neuron with activation $\sigma$ and normalized Gaussian initialization.

\begin{definition}
Let $f:\mathbb{R}^n\to\mathbb{R}$,
$\sigma:\mathbb{R}\to\mathbb{R}$ and let
$P_\cX$ be a distribution on $\mathbb{R}^n$.
Then, we abuse the notation and write
\begin{align}
    \INAL(f,\sigma):=
    \E_{w^n, b^n}\left[\Big(\E_{x\sim P_{\cX}}
    f(x)\sigma((w^n)^T x+b^n)\Big)^2\right] \;,
\end{align}
where $w^n$ is a vector of iid $\mathcal{N}(0,1/n)$ Gaussians
and $b^n$ is another independent $\mathcal{N}(0,1/n)$ Gaussian. In the following, for readability, we will write $w=w^n$ and $b=b^n$, omitting the dependence on $n$.
\end{definition}

In the following, we say that a function $f:\mathbb{N}\to\mathbb{R}_{\ge 0}$ is
\emph{noticeable} if there exists $c\in\mathbb{N}$ such that $f(n)=\Omega(n^{-c})$.
On the other hand, we say that $f$ is \emph{negligible} if $\lim_{n\to\infty} n^cf(n)=0$ for
every $c\in\mathbb{N}$
(
which we also write $f(n) = n^{-\omega_n(1)}$).

\begin{definition}[Weak learning]
Let $(f_n)_{n\in\mathbb{N}}$ be a sequence of functions such that
$f_n:\mathbb{R}^n\to\mathbb{R}$ and $(P_{n})$ a sequence
of probability distributions on $\mathbb{R}^n$.
Let $(A_n)$ be a family of randomized algorithms such that
$A_n$ outputs a function $\NN_n:\mathbb{R}^n\to\mathbb{R}$.
Then, we say that $A_n$ \emph{weakly learns} $f_n$ if the function
\begin{align}
    g(n):= \big|\E_{x\sim\mathcal{P}_n, A_n} [f_n(x) \cdot \NN_n(x)]\big|
\end{align}
is noticeable.
\end{definition}
In this paper, we follow the example of~\cite{abbe2020poly} and focus
on \emph{Boolean} functions with inputs
and outputs in $\{\pm1\}$.
We
consider sequences of Boolean functions $f_n:\{\pm 1\}^n\to\{\pm 1\}$,
with the uniform input distribution
$\cU^n$, meaning that if $x\sim\cU^n$, then for all $i \in [n]$,
$x_i \overset{iid}{\sim} \Rad(1/2) $.
We focus on fully connected
neural networks with activation function $\sigma$, and trained by noisy GD --- this means GD where the gradient's magnitude per the precision noise is polynomially bounded, as commonly considered in statistical query algorithms \cite{kearns1998efficient,blum1994weakly} and GD learning \cite{abbe2020poly,quantifying,abbe2021power}; see \Cref{rem:as-explanation}
for a remainder of the definition.
We consider activation functions that satisfy the following conditions.


\begin{definition}[Expressive activation] \label{def:expressive_ab}
We say that a function $\sigma:\mathbb{R}\to\mathbb{R}$ is \emph{expressive} if it satisfies the following conditions:
\begin{itemize}
    \item[a)] $\sigma$ is measurable and \emph{polynomially bounded}
    i.e. there exists $C,c>0$ such that $|\sigma(x)| \leq Cx^c+C$
    for all $x\in\mathbb{R}$.
    \item[b)] Let the Gaussian smoothing of $\sigma$ be defined as $\Sigma(t) := \E_{Y \sim \cN(0,1)}[\sigma(Y+t)] $. For each $m \in \bN$ either $ \Sigma^{(m)}(0) \neq 0$ or $\Sigma^{(m+1)}(0) \neq 0$ (where $\Sigma^{(m)}$ denotes the $m$-th derivative of $ \Sigma$).
\end{itemize}
\end{definition}

\begin{remark}
\begin{itemize}
    \item[i)] Note that we have the identities
    $d_m=\frac{\Sigma^{(m)}(0)}{m!}$, and
    $\sigma=\sum_{m=0}^{\infty} d_mH_m$,
    where $H_m$ are the probabilist's Hermite polynomials.
    Therefore, an equivalent statement of the second condition
    in \Cref{def:expressive_ab} is that there are
    no two or more consecutive zeros in the Hermite expansion
    of $\sigma$.
    \item[ii)] Many functions
    are expressive, including
    ReLU and sign (see Appendix~\ref{app:common_act}
    for the proofs of those two cases).
    \item[iii)]
    On the other hand, it turns out that
    polynomials
    are \emph{not} expressive, as they do not satisfy point $b)$. This is necessary for our hardness results to hold, since for an activation function $P$ which is a
    polynomial of degree $k$ and $M$ a monomial
    of degree $k+1$ it can be checked that $\INAL(M, P)=0$, but constant-degree
    monomials are learnable by GD.
\end{itemize}
\end{remark}
Let us give one more definition before stating our main theorem.
\begin{definition}[N-Extension]
For a function
$f: \mathbb{R}^n \to \mathbb{R}$ and for $N > n$, we define its \emph{N-extension}
$\overline f: \mathbb{R}^N \to\mathbb{R}$ as
\begin{align}
    \overline f (x_1,x_2,...,x_n, x_{n+1},x_{n+2},..., x_{N}) = f(x_1,x_2,...,x_n).
\end{align}
\end{definition}

We can now state our main result which connects INAL and weak learning.

\begin{theorem}[Main theorem, informal]
\label{thm:main-informal}
Let $\sigma$ be an expressive activation function and $(f_n)$ a sequence of Boolean functions with uniform distribution on $\{\pm 1\}^n$.
If $\INAL(f_n, \sigma)$ is
negligible, then, for every $\epsilon>0$,
the $n^{1+\epsilon}$-extension of $f_n$ is \emph{not}
weakly learnable by $\poly(n)$-sized fully-connected neural networks
with iid initialization
and $\poly(n)$-number of steps of noisy gradient descent.
\end{theorem}

\begin{remark}
\Cref{thm:main-informal} says that Boolean functions
that have negligible correlation for \emph{some} expressive
activation and Gaussian iid initialization, cannot be learned by neural networks utilizing
\emph{any} activation on a fully-connected architecture and
any iid initialization.
\end{remark}



\begin{remark}
Consider a sequence of neural networks $(\NN_n)$ utilizing
an expressive activation $\sigma$.
We believe that the notion of $\INAL(f_n, \NN_n)$
is relevant to characterizing if
a family of Boolean functions $(f_n)$ is weakly learnable by noisy GD
on those neural networks.
On the one hand, if $\INAL(f_n, \NN_n)$ is noticeable, then
at initialization there exists a neuron from which a weak
correlation with $f_n$ can be extracted. Therefore, in a sense
weak learning is achieved at initialization.

On the other hand, assume additionally that the architecture
is such that there exists a neuron computing
$\sigma(w^Tx+b)$, where $x$ is the input and $(w,b)$ are initialized
as iid $\mathcal{N}(0,1/n)$ Gaussians. (In other words, there exists
a fully-connected neuron in the first hidden layer.)
Then, by definition of INAL, if $\INAL(f_n, \NN_n)$ is negligible,
then also $\INAL(f_n, \sigma)$ is negligible.
Accordingly, by \Cref{thm:main-informal},
an extension of $(f_n)$ is not weakly learnable.

While we do not have a proof, we suspect that a similar
property might hold also for some other architectures
and initializations.
\end{remark}


Note that we obtain hardness only for an extension of $f_n$, rather than
for the original function. Interestingly, in some settings GD can learn the function, while the $2n$-extension of the same
function is hard to learn\footnote{
For example, for the Boolean parity function $M_n(x)=\prod_{i=1}^n x_i$ with both the input distribution and the weight initialization iid uniform in $\{\pm 1\}$ and cosine activation~\cite{EnricDraft}.}.
However, we are not sure if such examples can be constructed
for the continuous Gaussian initialization that we consider.




\section{Formal Results}\label{sec:results}
In this section, we write precise statements of our theorems.
For this, we need a couple of more definitions.

\begin{definition}[Cross-Predictability]
Let $P_\cF$ be a distribution over functions
from $\mathbb{R}^n$ to $\mathbb{R}$ and
$P_{\cX}$ a distribution over $\mathbb{R}^n$. Then,
\begin{align}
\CP(P_\cF,P_\cX) = \E_{F,F'\sim P_\cF} [\E_{X\sim P_\cX} [ F(X) F'(X) ]^2 ]\;.
\end{align}
\end{definition}

\begin{definition}[Orbit]
For
$f:\mathbb{R}^n\to\mathbb{R}$ and a permutation
$\pi\in S_n$, we let
$(f\circ \pi)(x)=f(x_{\pi(1)},\ldots,x_{\pi(n)})$.
Then, we define the \emph{orbit} of $f$ as
\begin{align}
\orb(f) : = \{ f\circ\pi : \pi \in S_n\}\;.
\end{align}
\end{definition}




Let us now give the full statement of our main theorem.
\begin{theorem} \label{thm:mainthm}
Let $(f_n)$ be a sequence of Boolean functions
with $f_n:\{\pm 1\}^n\to\{\pm 1\}$ and $x\sim\cU^n$ and let $\sigma$ be an expressive activation.

If $\INAL(f_n, \sigma)$ is negligible, then, for every
$\epsilon>0$, the cross predictability
$\CPMacro$ is negligible,
where $N=n^{1+\epsilon}$ and
$\orb(\overline{f_n})$ denotes (uniform distribution on)
the orbit of the $N$-extension of $f_n$.

More precisely,
if $\INAL(f_n,\sigma)=O(n^{-c})$, then
$\CPMacro=O(n^{-\frac{\epsilon}{1+\epsilon}(c-1)})$.
\end{theorem}

Applying~\cite{abbe2020poly}[Theorem 3] to \cref{thm:mainthm}
implies the following corollary. We refer to \cref{app:fully-connected}
for additional clarifications on the notion of a fully connected
neural net.

\begin{corollary}\label{cor:learning}
Let $f_n$ and $\sigma$ be as in \cref{thm:mainthm} with negligible $\INAL(f_n,\sigma)$ and let $\epsilon>0$ with
$N=n^{1+\epsilon}$ and $\overline{f_n}$ denote the
$N$-extension of $f_n$.

Let $\NN=(\NN_n)$ be any sequence of
fully connected neural nets
of polynomial size.
Then, for any iid initializaton, and
any polynomial bounds on the learning rate,
learning time $T=(T_n)$, noise level and overflow range, the noisy GD algorithm
after $T$ steps of training outputs a neural net
$\NN^{(T)}$ such that the correlation
\begin{align}
    g(n):=\big|\E_{\NN^{(T)}}\langle\NN^{(T)},\overline{f_n}\rangle
    \big|
\end{align}
is negligible.

More precisely,
if $\INAL(f_n,\sigma)=O(n^{-c})$, then for a noisy GD run for $T$
steps on a fully connected neural network with $E$ edges, with learning rate
$\gamma$, overflow range $A$ and noise level $\tau$
it holds that
\begin{align}
    g(n) = O\left(\frac{\gamma T\sqrt{E}A}{\tau}\cdot n^{-\frac{\epsilon}{4(1+\epsilon)}(c-1)}\right)\;.
\end{align}
\end{corollary}

\begin{remark}\label{rem:as-explanation}
In the result above, the neural net can have any
feed-forward architecture with layers of fully-connected
neurons and any activation such that the gradients
are almost surely well-defined.
The initialization can be iid
from any distribution (which can depend on $n$). We remark that the result of Corollary~\ref{cor:learning} can be strengthen to apply to any initialization such that the distribution of the weights in the first layer is invariant under permutations of input neurons. We refer to~\cref{app:fully-connected} for more details.

The algorithm considered is noisy gradient descent\footnote{In fact, it can be SGD with batch size $m$ for
large enough $m$.}
using any differentiable loss function, meaning that at every step
an iid $\mathcal{N}(0,\tau^2)$ noise vector is added
to all components of the gradient, where $\tau$ is called
the \emph{noise level}.
Furthermore,
every component of the gradient during the execution
of the algorithm whose evaluation
exceeds the \emph{overflow range} $A$ in absolute value
is clipped to $A$ or $-A$, respectively. This covers in particular the bounded `precision model' of \cite{abbe2021power}. 

For the purposes of function $g(n)$, it is assumed that
the neural network outputs a guess in $\{\pm 1\}$
using any form of thresholding (eg., the sign function)
on the value of the output neuron.
See~\cite{abbe2020poly}[Section~2.3.1].
\end{remark}

\section{Proof of Main Theorem}
In this section we sketch the proof of \Cref{thm:mainthm}. We first
state basic definitions from Boolean function analysis,
then we give a short outline of the proof, and then
we state main propositions used in the proof. Finally, we show how the propositions are combined to prove Theorem~\ref{thm:mainthm} and Corollary~\ref{cor:learning}.
Further proofs and details are in the appendices.

We introduce some notions of Boolean analysis, mainly taken from Chapters 1,2 of~\cite{o'donnell_2014}.
For every $f:\{\pm 1\}^n\to\mathbb{R}$ we denote its Fourier expansion as
\begin{align}
    f(x) = \sum_{S \subseteq [n]} \hat f(S)  M_S(x),
\end{align}
where $M_S(x) = \prod_{i \in S} x_i$ are the standard Fourier basis elements and $\hat f(S) $ are the Fourier coefficients of $f$, defined as $\hat f(S) = \langle f,M_S\rangle$. We denote by
\begin{align}
    W^{k} (f) & = \sum_{S: |S| = k} \hat f(S)^2\\
    W^{\leq k} (f) & = \sum_{S: |S| \leq k} \hat f(S)^2\;,
\end{align}
the total weight of the Fourier coefficients of $f$ at degree $k$ (respectively up to degree $k$).
\begin{definition}[High-Degree]
We say that a family of functions $f_n:\{ \pm 1\}^n \to \mathbb{R}$  is ``high-degree'' if for any fixed $k$, $W^{\leq k} (f_n) $ is negligible.
\end{definition}

\paragraph{Proof Outline of \Cref{thm:mainthm}.}
\begin{enumerate}
    \item We initially restrict our attention to the basis Fourier elements, i.e. the monomials $M_S(x) := \prod_{i \in S} x_i$ for $S \in [n]$. We consider the single-neuron
    alignments $\INAL(M_S,\sigma)$ for expressive activations.
    We prove that these INALs are noticeable for constant degree monomials (Proposition~\ref{prop:monomials}).
    \item For a general $f:\{\pm1 \}^n \to \bR$ we show that the initial alignment between $f$ and a single-neuron architecture can be computed from its Fourier expansion (Proposition~\ref{prop:general_functions}). As a consequence,
    for any expressive $\sigma$,
    if $ \INAL(f,\sigma)$ is negligible, then $f$ is high-degree (Corollary~\ref{cor:INAL_small_high_degree}).
    \item
    We construct the extension of $f$ and take its orbit $\orb(\overline{f})$.
    Since the extension has a sparse structure of its Fourier coefficients,
    that guarantees
    that the cross-predictability of $\orb(\overline{f})$ is negligible
    (Proposition~\ref{prop:CP_low}).
    \item In order to prove \Cref{cor:learning}, we invoke the lower bound of~\cite{abbe2020poly}
    (Theorem~\ref{thm:AS20}) applied to the class $\orb(\overline{f})$ .
\end{enumerate}

A crucial property of the expressive activations is that they correlate
with constant-degree monomials. To emphasize this, we introduce another
definition.
\begin{definition}
An activation $\sigma$ is \emph{correlating} if for every $k$,
the sequence $\INAL(M_k,\sigma)$ is noticeable, where we think of
$M_k(x)=\prod_{i=1}^k x_i$ as a sequence of Boolean functions for every input dimension $n\ge k$.

Furthermore, if there exists $c$ such that for every $k$
it holds $\INAL(M_k,\sigma)=\Omega(n^{-(k+c)})$,
then we say that $\sigma$ is \emph{$c$-strongly correlating}.
\end{definition}

\begin{proposition}
\label{prop:monomials}
If $\sigma$ is expressive (according to Definition~\ref{def:expressive_ab}),
then it is 1-strongly correlating.
\end{proposition}
The proof of \Cref{prop:monomials} is our main technical contribution.
Since the magnitude of the correlations is quite small
(in general, of the order $n^{-k}$ for monomials of degree $k$),
careful calculations are required to establish
our lower bounds.

In fact, we conjecture that any polynomially bounded function that is not a polynomial (almost everywhere) is correlating.

Then, we show that $\INAL(f,\sigma)$ decomposes
into monomial INALs according to its Fourier coefficients:

\begin{proposition}
\label{prop:general_functions}
For any $f:\{\pm 1\}^n\to\mathbb{R}$ and any activation $\sigma$,
\begin{align}
    \INAL(f,\sigma) := \sum_{T \in [n]}\hat f(T)^2 \INAL (M_T, \sigma)\;.
\end{align}
\end{proposition}

As a corollary, functions with negligible INAL
on correlating activations are high-degree:

\begin{corollary}
\label{cor:INAL_small_high_degree}
Let $\sigma$ be an activation with
$\INAL(M_{k'}, \sigma)=\Omega(n^{-k_0})$
for $k'=0,1,\ldots,k$. Then,
$W^{\le k}(f_n)\le\INAL(f_n, \sigma) O(n^{k_0})$.

In particular, if $\sigma$ is correlating
and $\INAL(f_n, \sigma)$ is negligible,
then $(f_n)$ is high degree.
\end{corollary}

Finally, the cross-predictability of $\orb(\overline{f_n})$
is negligible for high degree functions.

\begin{proposition} \label{prop:CP_low}
Let $\epsilon>0$ and $(f_n)$ a family of Boolean functions.
Let $(\overline{f_n})$ denote the family of $N$-extensions
of $f_n$ for $N=n^{1+\epsilon}$,
and consider the uniform distribution on its orbit.

If $(f_n)$ is high degree, then
$\CPMacro$ is negligible.
Furthermore, if for some universal $c$ and every fixed $k$ it holds
$W^{\le k}(f_n)=O(n^{k-c})$, then
$\CPMacro=
O(n^{-\frac{\epsilon}{1+\epsilon}\cdot c})$.
\end{proposition}

\begin{theorem}[\cite{abbe2020poly}, informal] \label{thm:AS20}
If the cross-predictability of a class of functions
is negligible, then noisy GD cannot learn it in poly-time.
\end{theorem}

We provide here an outline of the proof of Proposition~\ref{prop:monomials}, and refer to Appendix~\ref{app:proof_monomials} for a detailed proof. We further prove Proposition~\ref{prop:CP_low}
and
Theorem~\ref{thm:mainthm}.
The proofs of the remaining results are in appendices.

\paragraph{Proof of Proposition~\ref{prop:monomials} (outline).}
The main goal of the proof is to estimate the dominant term (as $n$ approaches infinity) of $\INAL(M_k,\sigma)$, and show that it is indeed noticeable, for any fixed $k$.
We initially use Jensen inequality to lower bound the INAL with the following
\begin{align} \label{eq:CS_outline}
    \INAL(M_k,\sigma ) \geq \E \left[ \E_{|\theta|,x}\left[M_k(x) \sigma({w}^T x+b) \mid \sign(\theta) \right]^2  \right],
\end{align}
where for brevity we denoted $\theta =(w,b)$, $|\theta|$ and $\sign(\theta)$ are $(n+1)$-dimensional vectors such that $|\theta|_i = | \theta_i|$ and $ \sign(\theta)_i = \sign(\theta_i)$, for all $i \leq n+1$. By denoting $|w|_{> k}, x_{> k}$ the coordinates of $|w|$ and $x$ respectively that do not appear in $M_k$, and by $G: = \sum_{i=1}^k w_i x_i + b$ we observe that
\begin{align}
    \E_{|w|_{> k}, x_{> k}} [\sigma(w^T x+b)  ] = \E_{Y \sim \cN(0, 1-\frac kn) }\left[ \sigma(G + Y) \right],
\end{align}
since $\sum_{i=k+1}^n w_i x_i$ is indeed distributed as $\cN(0, 1-\frac kn) $.
We call the RHS the ``$n$-Gaussian smoothing'' of $\sigma$ and we denote it by $\Sigma_n (z):= \E_{Y \sim \cN(0, 1-\frac kn) }\left[ \sigma(z + Y) \right]$. We will compare it to
the ``ideal'' Gaussian smoothing denoted by
$\Sigma(z):=\E_{Y\sim\cN(0,1)}[\sigma(z+Y)]$.

For polynomially bounded $\sigma$, we can prove that $\Sigma_n$ has some nice properties (see Lemma~\ref{def:sigma-properties}), specifically it is $C^{\infty}$ and polynomially bounded and it uniformly converges to $\Sigma$ as $n \to \infty$. These properties crucially allow to write $\Sigma_n$ in terms of its Taylor expansion around $0$, and bound the coefficients of the series for large $n$. In fact, we show that there exists a constant $P>k$, such that if we split the Taylor series of $\Sigma_n$ at $P$ as
\begin{align}
 \Sigma_n(G) = \sum_{\nu=0}^P a_{\nu,n} G^\nu + R_{P,n}(G),
\end{align}
(where $a_{\nu,n}$ are the Taylor coefficients and $ R_{P,n}$ is the remainder in Lagrange form), and take the expectation over $|\theta|_{\leq k}$ as:
\begin{align}
    \E_{|\theta|_{\leq k},x_{\leq k}} \left[ M_k(x) \Sigma_n(G) \right]  &= \sum_{\nu =0}^P a_{\nu,n}  \E_{|\theta|_{\leq k},x_{\leq k}} \left[ M_k(x) G^{\nu} \right] +  \E_{|\theta|_{\leq k},x_{\leq k}} \left[ R_{P,n}(G) \right]\\
    & =: A + B\label{eq:taylor_split_main},
\end{align}
then $A$ is $\Omega(n^{-P/2})$ (Proposition~\ref{prop:dominant_term}), and $B$ is $O(n^{-P/2-1/2})$ (Proposition~\ref{prop:error_term}), uniformly for all values of $\sign(\theta)$. For $A$ we use the observation that $\E_{|\theta|_{\leq k},x_{\leq k}} \left[ M_k(x) G^{\nu} \right]=0$ for all $\nu <k$ (Lemma~\ref{lem:moment-formula}), and the fact that $|a_{P,n}| >0$ for $n$ large enough (due to hypothesis b in Definition~\ref{def:expressive_ab} and the continuity of $\Sigma_n$ in the limit of $n \to \infty$, given by Lemma~\ref{def:sigma-properties}). For $B$, we combine the concentration of Gaussian moments and the polynomial boundness of all derivatives of $\Sigma_n$.


Taking the square of~\eqref{eq:taylor_split_main} and going back to~\eqref{eq:CS_outline}, one can immediately conclude that $\INAL(M_k,\sigma)$ is indeed noticeable.

\subsection{Proof of Proposition~\ref{prop:CP_low}}
Let $f=f_n$ and
let $\hat f$ be the Fourier coefficients of the original function $f$, and let $\hat h$ be the coefficients of the augmented function $\bar f$. Recall
that $\bar f:\{\pm 1\}^{N} \to \{\pm 1\}$ is such that $\bar f(x_1,...,x_n,x_{n+1},...,x_{N}) = f(x_1,...,x_n)$. Thus, the Fourier coefficients of $\bar f$ are
\begin{align}\hat h(T)&=\begin{cases}
    \hat f(T) & \text{ if } T \subseteq [n],\\
    0 & \text{ otherwise}.
\end{cases}\end{align}
Let us proceed to bounding the cross-predictability.
Below we denote by $\pi$ a random permutation of $N$ elements:
\begin{align}
   \CPMacro &= \E_{\pi}\left[ \E_x \left[ \bar f(x) \bar f(\pi(x))\right]^2\right]\\
    & = \E_{\pi}\left[\left( \sum_{T\subseteq [N]} \hat h(T) \hat h  (\pi(T)) \right)^2\right] \\
    & = \E_{\pi}\left[\left( \sum_{T \subseteq [n]} \hat f(T) \hat h (\pi(T))  \cdot \mathds{1}\left(\pi(T) \subseteq [n]\right) \right)^2\right] \\
    & \overset{C.S}{\leq} \E_{\pi}\left[
    \left(\sum_{S\subseteq[n]} \hat h(\pi(S))^2  \right) \right.\cdot \left.\left( \sum_{T\subseteq[n]} \hat f(T)^2 \mathds{1}\left(\pi(T) \subseteq [n]\right) \right)
    \right]\\
    & \leq \sum_{T \subseteq [n]} \hat f(T)^2 \cdot  \pr_\pi \left(\pi(T) \subseteq [n]\right).
\end{align}
Now, for any $k$ we have
\begin{align}
    \CPMacro &\leq \sum_{T: |T|< k} \hat f(T)^2 \cdot  \pr_\pi \left(\pi(T) \subseteq [n]\right)  + \sum_{T: |T| \ge k} \hat f(T)^2 \cdot  \pr_\pi \left(\pi(T) \subseteq [n]\right) \\
    & \leq W^{< k}(f) + \pr_\pi \left(\pi(T) \subseteq [n] \mid |T| = k \right)\;,\label{eq:02}
\end{align}
where the second term in~\eqref{eq:02} is further bounded by
(recall that $N = n^{1+\epsilon}$):
\begin{align}\label{eq:03}
    \pr_\pi \left(\pi(T) \subseteq [n] \mid |T| = k \right) & =
    \frac{\binom{n}{k}}{\binom{N}{k}} \\
    & \leq \frac{\left(\frac{ne}{k}\right)^{k}}{\left(\frac{N}{k}\right)^{k}} \\
    &= e^{k} \frac{n^{k}}{N^{k}} = e^{k} n^{-\epsilon\cdot k}.
\end{align}
Accordingly, for any $k\in\mathbb{N}_{>0}$ it holds that
\begin{align}\label{eq:04}
    \CPMacro\le
    W^{< k}(f)+e^k n^{-\epsilon k}\;.
\end{align}
Now, if $(f_n)$ is a high degree sequence of Boolean functions,
then $W^{< k}(f)$ is negligible for every $k$, and therefore
the cross-predictability in~\eqref{eq:03} is $O(n^{-k})$ for every
$k$, that is the cross-predictability is negligible as we claimed.

On the other hand, if for some $c$ and every $k$ it holds that
$W^{\le k}(f_n)=O(n^{k-c})$, then we can choose
$k_0:=\frac{c}{1+\epsilon}$ and apply~\eqref{eq:04} to get
$\CPMacro=O(n^{-\frac{\epsilon}{1+\epsilon}\cdot c})$.

\subsection{Proof of Theorem \ref{thm:mainthm}}
\label{app:main-proof}

Let $\sigma$ be an expressive activation and let $(f_n)$ be a sequence
of Boolean functions with negligible $\INAL(f_n,\sigma)$.
By \Cref{prop:monomials}, $\sigma$ is correlating,
and by \Cref{cor:INAL_small_high_degree} $(f_n)$ is high-degree.
Therefore, by \Cref{prop:CP_low}, the cross-predictability
$\CPMacro$ is negligible.

For the more precise statement,
let $(f_n)$ be a sequence of Boolean functions
with $\INAL(f_n,\sigma)=O(n^{-c})$.
By \Cref{prop:monomials}, $\sigma$ is 1-strongly correlating.
That means that for every $k$ we have
$\INAL(M_k,\sigma)=\Omega(n^{-(k+1)})$.
By \cref{cor:INAL_small_high_degree}, for every $k$ it holds
$W^{\le k}(f_n)=O(n^{k+1-c})$. Finally, applying \Cref{prop:CP_low},
we have that $\CPMacro=O(n^{-\frac{\epsilon}{1+\epsilon}(c-1)})$.

\begin{figure*}[ht!]
\centering
\includegraphics[width = \textwidth]{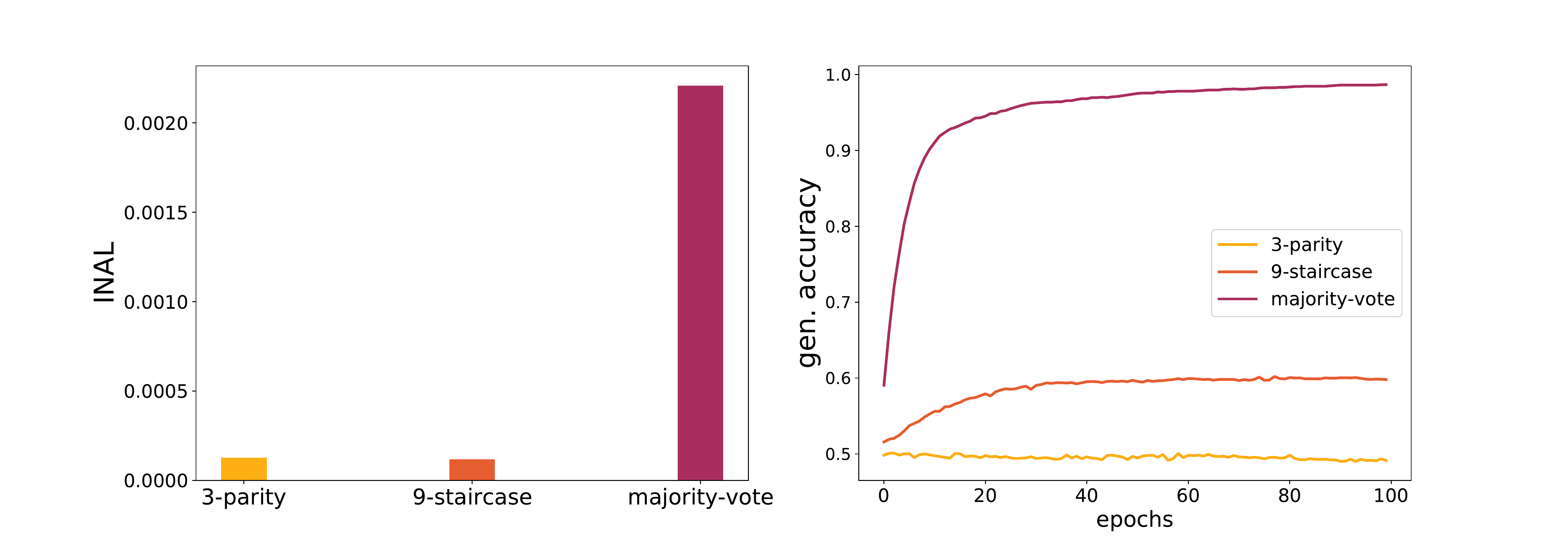}
\caption{Comparison of INAL and generalization accuracy for three Boolean functions. On the left, we estimate the INAL between each target function and a 2-layers ReLU fully connected neural network with normalized gaussian initialization. On the right, we train the network to learn each target function with SGD with batch 1000 for 100 epochs. We observe that low INAL is bad news.}
\label{experiments_bool}
\end{figure*}
\begin{figure*}[ht!]
\includegraphics[width = \textwidth]{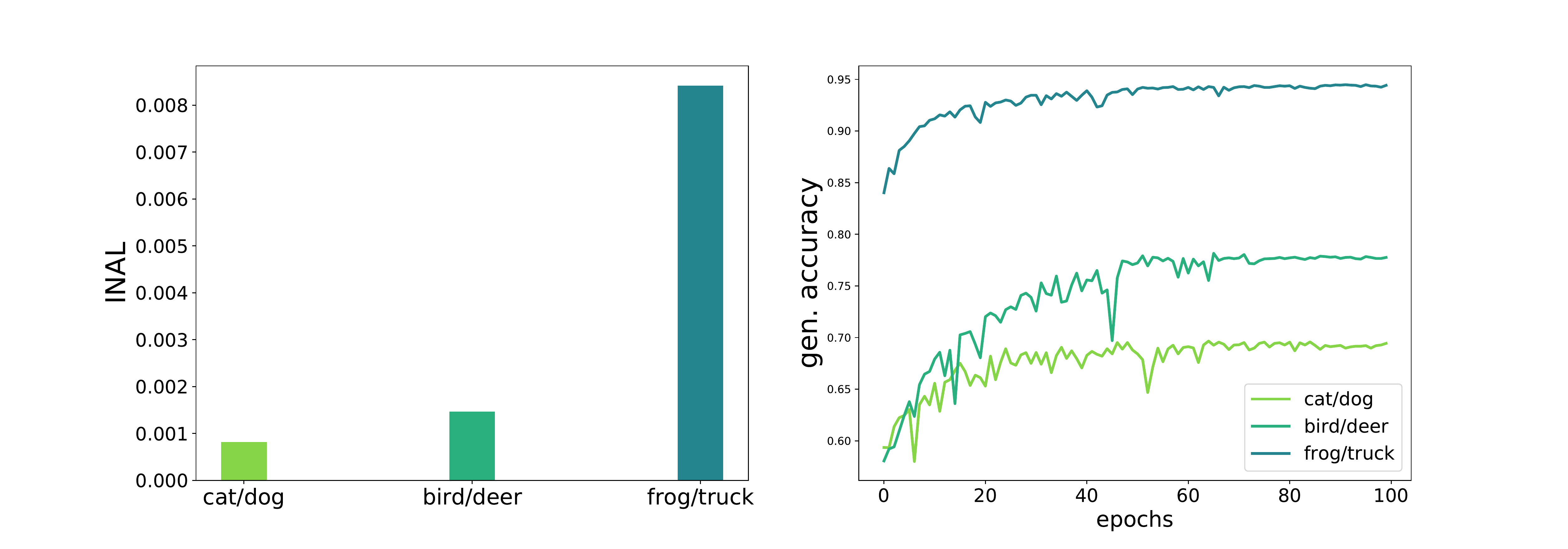}
\caption{Comparison of INAL and generalization accuracy for binary classification in the CIFAR dataset.
On the left, we estimate the INAL between the neural network and the target function associated to each task. On the right we train a CNN with 1 VGG block with SGD with batch size 64 for 100 epochs.
We observe that a significant difference in the INAL corresponds to a significant difference in the generalization accuracy achieved.}
\label{fig:experiments_CIFAR}


\end{figure*}

\section{Experiments}\label{exper}
In this section we present a few experiments to show how the INAL can be estimated in practice. Our theoretical results connect the performance of GD to the Fourier spectrum of the target function. However, in applications we are usually given a dataset with data points and labels, rather than an explicit target function, and it may not be trivial to infer the Fourier properties of the function associated to the data. Conveniently, the INAL can be estimated with sufficient datapoints and labels, and do not need an explicit target.

\paragraph{Experiments on Boolean functions.} In our first experiment, we consider three Boolean functions, namely the majority-vote over the whole input space (${\rm Maj}_n(x) := \sign(\sum_{i=1}^n x_i)$),
a $9$-staircase (${\rm S_9(x) := \sign(x_1+x_1x_2+x_1x_2x_3+...+x_1x_2x_3...x_9)}$ and a $3$-parity ($M_3(x) = \prod_{i=1}^3 x_i$), on an input space of dimension $100$. We take a 2-layer fully connected neural network with ReLU activations and normalised Gaussian iid initialization (according to the setting of our theoretical results), and we train it with SGD with batch-size 1000 for 100 epochs, to learn each of the three functions. On the other hand, we estimate the INAL between each of the three targets and the neural network, through Monte-Carlo. Our observations confirm our theoretical claim, i.e. that low INAL is bad news. In fact, for the 3-parity and the 9-staircase, that have very low INAL ($\sim$1/20 of the majority-vote case), GD does not achieve good generalization accuracy after training (Figure~\ref{experiments_bool}).

\paragraph{Experiments on real data.}
Given a dataset $D=(x_m,y_m)_{m \in [M]}$, where $x_m \in \bR^n$, and $y_m \in \bR$, and given a randomly initialized neural network $\NN_{\Theta^0}$ with $\Theta^0 $ drawn from some distribution, we can estimate the initial alignment between the network and the target function associated to the dataset as
\begin{align} \label{eq:empirical_INAL}
     \max_{v \in V_{NN}} \E_{\Theta^0}\left[ \left(\frac{1}{M}\sum_{m =1}^M y_m \cdot  \NN_{\Theta^0}^{(v)}(x_m) \right)^2 \right],
\end{align}
where the outer expectation can be performed through Monte-Carlo approximation.

We ran experiments on the CIFAR dataset.
We split the dataset into 3 different pairs of classes, corresponding to 3 different binary classification tasks (specifically cat/dog, bird/deer, frog/truck). We take a CNN with 1 VGG block and $\ReLU$ activation, and for each task, we train the network with SGD with batch-size 64, and we estimate the INAL according to~\eqref{eq:empirical_INAL}.
We notice that also in this setting (not covered by our theoretical results), the INAL and the generalization accuracy present some correlation, and a significant difference in the INAL corresponds to a significant difference in the accuracy achieved after training.
This may give some motivation to study the INAL beyond the fully connected setting.

\section{Conclusion and Future Work}
There are several directions that can follow from this work. The most relevant would be to extend the result beyond fully connected architectures. As mentioned before, we suspect that our result can be generalized to all architectures that contain a fully connected layer anywhere in the network.
Another direction would be to extend the present work to other continuous distributions of intitial weights (beyond gaussian). As a matter of fact, in the setting of iid gaussian inputs (instead of Boolean inputs), our proof technique extends to all weight initialization distributions with zero mean and variance $O(n^{-1})$. However, in the case of Boolean inputs
that we consider in this paper, this may not be a trivial extension.
Another extension on which we do not touch here
are non-uniform input distributions.

\paragraph{Acknowledgements}
We thank Peter Bartlett for a helpful discussion.

\bibliography{arxiv_revision}
\bibliographystyle{alpha}

\newpage
\appendix
\onecolumn

\section{Proof of Proposition~\ref{prop:monomials}}
\label{app:proof_monomials}

For an activation $\sigma:\bR \to \bR$, we denote its $v$-Gaussian smoothing as
\begin{align}
    \Sigma_v(t) := \E_{Y \sim \cN(0,v)}[\sigma(Y+t)].
\end{align}
We also write $\Sigma:=\Sigma_1$ for brevity.
As mentioned, we will be working with functions
that are \emph{polynomially bounded}, ie., such that
there exists a polynomial $P$ with $|\sigma(x)|< P(x)$ holding for all
$x\in\mathbb{R}$. We will use the fact that such polynomial can be assumed
wlog to be of the form $|\sigma(x)|<Cx^{\ell}+C$ for some $C>0$ and
$\ell\in\mathbb{N}_{\ge 0}$ (since any polynomial can be upper bounded by
a polynomial of such form).
Note that if $\sigma$ is a measurable, polynomially bounded function,
then $\Sigma_v$
is well defined for every $v>0$.

We now state the intermediate step in the proof of \Cref{prop:monomials}:
\begin{lemma}[Conditions on $\Sigma$ and $\Sigma_v$]\label{def:sigma-properties}
If $\sigma$ is a measurable, polynomially bounded function, then it satisfies the
following conditions:
\begin{enumerate}
    \item[i)] $\Sigma_v\in C^{\infty}(\bR)$
    for every $v>0$;
    \item[ii)] For every $k\in\mathbb{N}_{\ge 0}$
    and $v>0$,
    $ \Sigma_{v}^{(k)} (t) := \frac{\partial^k}{\partial t^k}\Sigma_{v}(t)$ is polynomially bounded. Furthermore,
    this bound is uniform,
    that is, $|\Sigma_v^{(k)}(t)|<Ct^{\ell}+C$
    holds for every $t\in\mathbb{R}$ and
    every $1/2\le v\le 1$, for some $C,\ell$
    that do not depend on $v$.
    \item[iii)] For all $k\in\mathbb{N}_{\ge 0}$, it holds
    $|\Sigma^{(k)}_{1-\epsilon} (0) - \Sigma^{(k)}(0) | = O (\epsilon)$.
\end{enumerate}
\end{lemma}


\Cref{def:sigma-properties} is then used in the proof of

\begin{lemma}\label{prop:smooth-is-correlating}
Let $\sigma$ be expressive (according to Definition~\ref{def:expressive_ab}). Then, for every $k\ge 0$ and $P\ge k$ such that
$\Sigma^{(P)}(0)\ne 0$, it holds that
$\INAL(M_k, \sigma)=\Omega(n^{-P})$.
\end{lemma}

In particular, from \Cref{prop:smooth-is-correlating} it follows that if $\sigma$ is expressive, then it is correlating.
Furthermore, since by condition b) in \Cref{def:expressive_ab}
for every $k$ we have $\Sigma^{(k)}(0)\ne 0$
or $\Sigma^{(k+1)}(0)\ne 0$, by \Cref{prop:smooth-is-correlating}
it holds $\INAL(M_k,\sigma)=\Omega(n^{-(k+1)})$ and
$\sigma$ is $1$-strongly correlating.



In the following subsections we prove
\Cref{def:sigma-properties} and \Cref{prop:smooth-is-correlating}.

\subsection{Proof of Lemma~\ref{def:sigma-properties}}
In the following let $\phi_v$ denote the density function of $\cN(0, v)$,
ie., $\phi_v(t)=\frac{1}{\sqrt{2v\pi}}\exp\left(-\frac{t^2}{2v}\right)$.
Note the relation to the standard Gaussian density $\phi=\phi_1$
where $\phi_v(t)=\frac{1}{\sqrt{v}}\phi(t/\sqrt{v})$.

We recall some useful facts about the derivatives of $\phi_v$. First, it
is well known that for $\phi$
it holds $\phi^{(k)}(t)=P_k(t)\phi(t)$ for some polynomial $P_k$ of
degree $k$. This formula extends to $\phi_v$ according to
\begin{align}\label{eq:13}
    \phi_v^{(k)}(t)&=\frac{1}{\sqrt{v}}\frac{\mathrm{d}^k}{\mathrm{d}t^k}
    \phi(t/\sqrt{v})
    =v^{-k/2-1/2}\phi^{(k)}(t/\sqrt{v})
    =v^{-k/2-1/2}P_k(t/\sqrt{v})\phi(t/\sqrt{v})\\
    &=v^{-k/2}P_k(t/\sqrt{v})\phi_v(t)\;.
\end{align}

i) Let us write $\phi_v(t) = (\phi_{v/2} \ast  \phi_{v/2}) (t)$
where
$\ast $ denotes the convolution in $\bR$, i.e. $(g \ast h) (y) = \int_\bR g(x) h(y-x) dx $.
Thus,
\begin{align}\label{eq:14}
    \Sigma_v = \sigma\ast\phi_v = (\sigma \ast  \phi_{v/2} ) \ast \phi_{v/2}.
\end{align}
Now, $\sigma \ast  \phi_{v/2} $ is in $L_1(\bR)$, since $\sigma$ is measurable
and polynomially bounded. Furthermore, $\phi_{v/2}$ is in $L_1(\bR)$ and $C^{\infty}(\bR)$. Therefore, by formulas for
derivatives of convolution, $\Sigma_v\in C^\infty(\bR)$.

ii) Let us start with the claim that $\Sigma_v^{(k)}$ is polynomially bounded
for every $v$ and $k$. For that, we recall some facts.
First, it is easy to establish by direct computation that
if $\sigma$ is polynomially bounded, then $\Sigma_v=\sigma\ast\phi_v$
is also polynomially bounded. Furthermore, if $P$ is any polynomial,
then also $\sigma\ast(P\phi_{v})$ is polynomially bounded
(this can be seen, eg., by observing that for every $P$ and
every $v'>v$ there exists $C$ such that
$|P\phi_v|\le C\phi_{v'}$).

Accordingly, using~\eqref{eq:13} and~\eqref{eq:14} we have that
\begin{align}
    \Sigma_v^{(k)}=(\sigma \ast  \phi_{v/2} ) \ast \phi_{v/2}^{(k)}
    =(\sigma\ast\phi_{v/2})\ast (P_{k,v}\phi_{v/2})
\end{align}
is polynomially bounded.

Let us move to the second claim with uniform bound.
For that let $k\ge 0$ and $1/2\le v\le 1$.
Let $v':=v-1/4$ and note that $1/4\le v'\le 3/4$.
Then, we have
the sequence of bounds on functions which hold pointwise:
\begin{align}
    |\Sigma_v^{(k)}|
    &=\left|(\sigma\ast\phi_{1/4})\ast\phi_{v'}^{(k)}\right|
    \le C_1\left(
    |\sigma\ast\phi_{1/4}|\ast |P_k(x/\sqrt{v'})|\phi_{v'}
    \right)\\
    &\le C_1\left(
    |\sigma\ast\phi_{1/4}|\ast (C_2+C_2(x/\sqrt{v})^{2\ell})
    \phi_{v'}
    \right)\\
    &\le C_3\left(
    |\sigma\ast\phi_{1/4}|\ast (C_4+C_4x^{2\ell})
    \phi
    \right)\;,
\end{align}
which is now bounded by a polynomial which does not depend on $v$.

iii) Recall,
\begin{align}
    \Sigma_v^{(k)} (0) = \int_{-\infty}^\infty (\phi_{v/2} \ast \sigma) (x) \cdot  \frac{\partial^k}{\partial t^k} \phi_{v/2}(x+t) \Big|_{t=0} dx,
\end{align}
where we denoted by $\phi_{v/2}^{(k)} $ the $k$-th derivative of $\phi_{v/2} $. Firstly, note that
\begin{align}
    \frac{\partial^k}{\partial t^k} \phi_{v/2}(x+t) \Big|_{t=0} = \frac{\partial^k}{\partial (x+t)^k} \phi_{v/2}(x+t)  \Big|_{t=0} = \phi_{v/2}^{(k)}(x) .
\end{align}
Let us give a formula for the k-th derivative of the Gaussian density:
\begin{align}
    \phi_v^{(k)} (x) = \phi_v(x) \cdot (-1)^k v^{-2k}\cdot \sum_{l=0}^k D_{l,k}  \left(\frac{x}{\sqrt{v}} \right)^{k-l},
\end{align}
where $D_{l,k}$ is a constant that does not depend on $v$, specifically
\begin{align}
D_{l,k}:= B_{(2k+l) \frac{1-(-1)^l}{2}} \cdot  2^{\frac{l}{2}}\cdot  \frac{ \Gamma(\frac{l+1}{2})}{\Gamma(\frac{1}{2})}\cdot  \cos \left(\frac{l \pi}{2} \right)
\end{align}
where $\Gamma(.)$ denotes the Gamma function and $B_n$ are the Bernoulli numbers. The exact values of the $D_{l,k}$ will not be relevant for this proof. Thus,
\begin{align}
    \Sigma_v^{(k)} (0): = \int_{-\infty}^\infty  (\phi_{v/2} \ast \sigma) (x) P_{v/2,k}(x) \phi_{v/2}(x) dx,
\end{align}
where we denoted $P_{v/2,k}(x) = (-1)^k v^{-2k}\cdot \sum_{l=0}^k D_{l,k}  \left(\frac{x}{\sqrt{v}} \right)^{k-l} $. On the other hand,
\begin{align}
    \Sigma_1^{(k)} (0): = \int_{-\infty}^\infty  (\phi_{v/2} \ast \sigma) (x) P_{1-v/2,k}(x) \phi_{1-v/2}(x) dx,
\end{align}
and
\begin{align}
    | \Sigma_v^{(k)}(0) - \Sigma_1^{(k)}(0) | = \Big|\int_{-\infty}^\infty  (\phi_{v/2} \ast \sigma) (x) \cdot  \left( P_{v/2,k}(x) \phi_{v/2}(x) -  P_{1-v/2,k}(x) \phi_{1-v/2}(x)  \right) \Big|.
\end{align}
We note that
\begin{align}
     P_{1-v/2,k}(x)=&\frac{(1-\frac{v}{2})^{-2k}}{(\frac{v}{2})^{-2k}} \left(\frac{v}{2}\right)^{-2k}  (-1)^k\\
     & \cdot \left( \sum_{l=0}^k D_{l,k} \left(\frac{x}{\sqrt{v/2}} \right)^{k-l} + D_{l,k} \left[ \left(\frac{x}{\sqrt{1-v/2}} \right)^{k-l} -  \left(\frac{x}{\sqrt{v/2}} \right)^{k-l}  \right]\right) \\
      = &\frac{(1-\frac{v}{2})^{-2k}}{(\frac{v}{2})^{-2k}} P_{v/2,k}(x)  \\
      & + \left(1- \frac{v}{2} \right)^{-2k} (-1)^k\sum_{l=0}^k D_{l,k} \left( \left(\frac{x}{\sqrt{1-v/2}} \right)^{k-l} -  \left(\frac{x}{\sqrt{v/2}} \right)^{k-l} \right).
\end{align}
Recalling $ \epsilon = 1-v$, and expanding for such $\epsilon$ we get
\begin{align}
  & \left(1+\frac{2 \epsilon}{1-\epsilon}\right)^{-2k} P_{v/2,k}(x) +  (1+\epsilon)^{-2k}\frac{(-1)^k}{2^{-2k}} \sum_{l=0}^k D_{l,k} x^{k-l} \frac{ (1-\epsilon)^{k-l}- (1+\epsilon)^{k-l}}{(1+ \epsilon )^{\frac{k-l}{2} }(1- \epsilon )^{\frac{k-l}{2} }}\\
    =& \left(1 -4k \frac{\epsilon}{1-\epsilon} + o(\epsilon) \right) P_{v/2,k}(x)\\
    &+ (1-2k\epsilon + o(\epsilon)) \frac{(-1)^k}{2^{-2k}} \sum_{l=0}^k D_{l,k} x^{k-l} \frac{-2 (k-l) \epsilon + (\epsilon) }{(1+\frac{k-l}{2}\epsilon  +o(\epsilon))(1-\frac{k-l}{2}\epsilon  +o(\epsilon))} \\
    =& \left(1- 4k \frac{\epsilon}{1-\epsilon}\right) P_{v/2,k}(x) + O(\epsilon) \cP_k(x),
\end{align}
where $\cP_k(x)$ is a polynomial in $x$ of degree $\leq k$. Moreover,
\begin{align}
    \phi_{1-v/2}(x) &= \frac{e^{- \frac{x^2}{v} }}{\sqrt{2 \pi v/2}} \cdot  \sqrt{\frac{v/2}{1-v/2}} \cdot  e^{-\frac{x^2}{2}\left( \frac{1}{1-\frac v2} - \frac{2}{v}\right) }\\
    & = \phi_{v/2} (x) \cdot \left(1 - \frac{2\epsilon}{1+\epsilon}  \right)^{1/2} \cdot e^{x^2 \frac{2\epsilon}{(1+\epsilon)(1-\epsilon)}}\\
    & = \phi_{v/2} (x) \cdot \left(1-\frac{\epsilon}{1+\epsilon} +o(\epsilon) \right) \cdot \left( 1+ x^2 \frac{2 \epsilon}{(1+\epsilon)(1-\epsilon)} + o(\epsilon) x^4  \right)\\
    & = \phi_{v/2} (x) \cdot \left( 1 +(x^2-1) O(\epsilon) \right).
\end{align}
Plugging these bounds in the previous expression, we get
\begin{align}
     | \Sigma_v^{(k)}(0) &- \Sigma_1^{(k)}(0) | \\
     & =\Big|\int_{-\infty}^\infty  (\phi_{v/2} \ast \sigma) (x) \cdot  \left( P_{v/2,k}(x) \phi_{v/2}(x) - P_{1-v/2,k}(x) \phi_{v/2} (x) \left( 1 +(x^2-1) O(\epsilon) \right) \right) \Big| \\
     & =\Big|\int_{-\infty}^\infty  (\phi_{v/2} \ast \sigma) (x)\phi_{v/2}(x)  \cdot  \left( P_{v/2,k}(x)- P_{1-v/2,k}(x)  \left( 1 +(x^2-1) O(\epsilon) \right) \right) \Big|  \\
     &= \Big|\int_{-\infty}^\infty  (\phi_{v/2} \ast \sigma) (x)  P_{v/2,k}(x) \phi_{v/2}(x) \cdot  \left(1-  (1- O(\epsilon) +(x^2-1) O(\epsilon)  ) + O(\epsilon) \cP_k(x) \right)  \Big|\\
     & =  O(\epsilon)\;.&\hfill\qed
\end{align}

\subsection{Proof of Lemma \ref{prop:smooth-is-correlating}}

Note that we only need to show that
$\INAL(M_k,\sigma)=\Omega(n^{-P})$ for the first
index $P$ such that $P\ge k$ and $\Sigma^{(P)}(0)\ne 0$.
By \Cref{def:expressive_ab}, we only need with two cases
$P=k$ and $P=k+1$. From now on, let us consider a fixed
pair of $k$ and $P$.

We denote by $x \in \{\pm 1\}^n$ the vector of all inputs, by $w \in \bR^n $ the vector of all weights and by $b \in \bR$ the bias. Additionally, we denote $\tau_i := \sign(w_i)$, and by $\tau \in \{\pm 1\}^n$ the vector of all weight signs. Recall that we consider $w_i,b \overset{iid}{\sim} \cN(0,\frac{1}{n})$ and that for $g,h: \{\pm 1\}^n \to \{ \pm 1\}$ and $\cU^n$ being the uniform distribution over the hypercube, we denote $ \langle g,h\rangle = \E_{x \sim \cU^n}[g(x) h(x)]$. We have
\begin{align}
\INAL(M_k,\sigma) &= \E_{w,b} \left[ \langle M_k,\sigma \rangle^2 \right]\\
    & = \E_{|w|,\tau,|b|, \sign(b)} \left[ \langle M_k,\sigma \rangle^2 \right]\\
    & \overset{(C.S.)}{\geq} \E_{\tau, \sign(b)} \left[
    \E_{|w|,|b|}\big[\langle M_k,\sigma \rangle\mid \tau, \sign(b)\big]^2
    \right], \label{eq:Cauchy-Schwartz}
\end{align}
where~\eqref{eq:Cauchy-Schwartz} follows by Cauchy-Schwartz inequality.
We will prove a lower bound on the inner expectation
$\left( \E_{|w|,|b|}  \langle M_k,\sigma \rangle \right)^2$ which is independent
of $\tau$ and $\sign(b)$. Accordingly, from now on consider $\tau$ and $\sign(b)$ to
be fixed at arbitrary values.

Let $T:=\{1,\ldots,k\}$ and denote by $x_T$ the coordinates of $x$ contained in $T$, and by $x_{\sim T} := x_{T^C} $ the coordinates of $x$ that are not contained in $T$ and hence do not appear in the monomial $M_T$. Similarly, we denote by $|w|_T, |w|_{\sim T}$ the coordinates of $|w|$ that appear (respectively do not appear) in set $T$. We proceed,
\begin{align}
    \E_{|w|,|b|}  \langle M_T,\sigma \rangle& = \E_{x,|w|,|b|} \left[  M_T(x) \cdot \sigma \left(\sum_{i\in [n]} w_i x_i +b\right) \right] \label{eq:inner_expectation}\\
    & = \E_{|w|_T,x_T,|b|}\left[ M_T(x) \cdot\E_{|w|_{\sim T},x_{\sim T}}    \sigma \left(\sum_{i\in [n]} w_i x_i +b\right) \right]
\end{align}
Observe that $ \sum_{i \not\in T} w_i x_i \sim \cN(0, \frac{n-k}{n})$, and denote $\Sigma_n(z) :=\Sigma_{1-\frac{k}{n}}(z)=\E_{Y \sim  \cN(0, \frac{n-k}{n})}[\sigma(z+Y)]$.
Moreover, let $G := \sum_{i \in T} w_i x_i+b$. Then,
\begin{align}\label{eq:06}
    \E_{|w|,|b|}  \langle M_T,\sigma \rangle& = \E_{|w|_T,|b|,x_T}\left[ M_T(x)
    \Sigma_n\left(G \right) \right].
\end{align}

Since, by condition i) in \Cref{def:sigma-properties},
function $\Sigma_n$ is $C^\infty$ and therefore $C^P$,
we apply Taylor's theorem with Lagrange remainder and write
\begin{align}
    \Sigma_n (z)= \sum_{\nu=0}^P a_{\nu,n} z^\nu + R_{P,n}(z),
\end{align}
where $a_{\nu,n}=\frac{\Sigma^{(\nu)}_n(0)}{\nu!}$ and
\begin{align}
R_{P,n}(z) = \frac{\Sigma_n^{(P+1)}(\xi_z)}{(P+1)!}z^{P+1} \qquad \text{ for some  } |\xi_z|\leq |z|.
\end{align}

Plugging this in~\eqref{eq:06}, we get
\begin{align}
    \E_{|w|,|b|}  \langle M_T,\sigma \rangle =
    \sum_{\nu =0}^P a_{\nu,n} \E_{|w|_T,|b|,x_T} \big[
    M_T(x)G^\nu \big]  + \E_{|w|_T,|b|,x_T} \big[ M_T(x) R_{P,n}\left(G\right)\big] \label{eq:taylor_two_terms}.
\end{align}
The following two propositions give the asymptotic characterization of the first and second term in~\eqref{eq:taylor_two_terms}.
\begin{proposition} \label{prop:dominant_term}
\begin{align}
 \sum_{\nu =0}^P a_{\nu,n} \E_{|w|_T,|b|,x_T} \big[M_T(x)G^\nu\big]=
 C(P)(-1)^{C'(\tau_T,\sign(b))}n^{-P/2}+O(n^{-P/2-1/2})\;.
\end{align}
where $C(P)\ne 0$ and $C'(\tau_T,\sign(b))\in\mathbb{Z}$ are constants
that do not depend on $n$.
\end{proposition}

\begin{proposition} \label{prop:error_term}
\begin{align}
\E_{|w|_T,|b|,x_T} \big[M_T(x)R_{P,n}\left(G\right)\big]
= O(n^{-P/2-1/2}).
\end{align}
\end{proposition}
Before proving Propositions~\ref{prop:dominant_term} and~\ref{prop:error_term}, let us see how
\Cref{prop:smooth-is-correlating} follows from them.
But this is clear: substituting into~\eqref{eq:taylor_two_terms}, we have
\begin{align}
    \left(\E_{|w|,|b|}  \langle M_T,\sigma \rangle \right)^2 =
    C(P)^2 n^{-P} + O(n^{-P-1/2})=\Omega(n^{-P})\;,
\end{align}
where the claimed bound does not depend on $\tau$ nor on $\sign(b)$.

\subsubsection{Proof of Proposition~\ref{prop:dominant_term}}
The main step for proving Proposition~\ref{prop:dominant_term} is the computation of $\langle M_T,G^\nu \rangle$, for $\nu \leq P$. This is summarized in the following formula.
\begin{lemma}\label{lem:moment-formula}
We have:
\begin{align}\label{eq:09}
\E_{|w|_T,|b|,x_T} M_T(x)G^{\nu}=\begin{cases}
0 &\text{if } \nu <k \vspace{0.5em}\\
C(\nu)(-1)^{C'(\tau_T,\sign(b))}n^{-\nu/2}
&\text{if }  \nu \ge k\;,
\end{cases}
\end{align}
where $C(\nu)>0$.
\end{lemma}
Let us first see how to finish the proof once \Cref{lem:moment-formula} is established.
Recall that $a_{\nu,n}=\frac{\Sigma^{(\nu)}_{1-k/n}(0)}{\nu!}$ and
let $a_\nu:=\frac{\Sigma^{(\nu)}(0)}{\nu!}$. We are considering a sum with $P+1$ terms,
so let $s_\nu:=a_{\nu,n}\E_{|w|_T,|b|,x_T}\big[M_T(x)G^\nu\big]$.
Accordingly, our objective is to show
that
\begin{align}
    \sum_{\nu=0}^P s_{\nu}=C(P)(-1)^{C'(\tau_T,\sign(b))}
    n^{-P/2}+O(n^{-P/2-1/2})\;.
\end{align}
We do that by considering the terms $s_\nu$ one by one.
For $\nu<k$, from \eqref{eq:09} we immediately have $s_\nu=0$.

For $k\le\nu<P$, by \Cref{def:expressive_ab}
recall that the only possible case is $P=k+1$
and $\Sigma^{(k)}(0)=0$.
Then, applying condition iii) from \Cref{def:sigma-properties},
\begin{align}
    |a_{\nu,n}|
    &=\left|\frac{\Sigma_{1-k/n}^{(\nu)}(0)-\Sigma^{(\nu)}(0)}{\nu!}\right|
    =O(n^{-1})\;,
\end{align}
which together with~\eqref{eq:09} gives $|s_v|=O(n^{-P/2-1/2})$.

Finally, for $\nu=P$, by assumption we have $a_P\ne 0$. Then,
by condition iii), we have $|a_{P,n}-a_P|=O(1/n)$ and~\eqref{eq:09}
gives us the correct form for $s_P$ and the whole expression.

All that is left is the proof of \Cref{lem:moment-formula}.

\begin{proof}[Proof of \Cref{lem:moment-formula}]
The proof proceeds by using the linearity of expectation and
independence and expanding the formula for $G^\nu$.
Recall that we assumed wlog that $T=\{1,\ldots,k\}$
and let $z_i:=w_ix_i$ for $i\le k$ and $z_{k+1}:=b$:
\begin{align}
    \E_{|w|_T,|b|,x_T}M_T(x)G^{\nu}
    &=\E_{|w|_T,|b|,x_T}\left(\prod_{i=1}^kx_i\right)
    \left(\sum_{i=1}^kw_ix_i+b\right)^\nu\\
    &=\sum_{I=(i_1,\ldots,i_{\nu})\in[k+1]^{\nu}}
    \E_{|w|_T,|b|,x_T}\left(\prod_{i=1}^kx_i\right)
    \left(\prod_{i\in I}z_i\right)\;.
    \label{eq:10}
\end{align}
Let us focus on a single term of the sum in~\eqref{eq:10}
for $I=(i_1,\ldots,i_{\nu})\in[k+1]^\nu$. For $j=1,\ldots,k+1$,
let $\alpha_j=\alpha_j(I):=|\{m:i_m=j\}|$. Accordingly, we
can rewrite a term from~\eqref{eq:10} as
\begin{align}
    \E_{|w|_T,|b|,x_T}&\left(\prod_{i=1}^kx_i\right)
    \left(\prod_{i\in I}z_i\right)\\
    &=\E_{|w|_T,|b|,x_T}
    \left(\prod_{i=1}^k w_i^{\alpha_i}x_i^{\alpha_i+1}\right)
    b^{\alpha_{k+1}}\\
    &=\left(\prod_{i=1}^k\tau_i^{\alpha_i}\right)\sign(b)^{\alpha_{k+1}}
    \left(\prod_{i=1}^k \E_{|w|_i}\big[|w|_i^{\alpha_i}\big]
    \cdot\E_{x_i}\big[x_i^{\alpha_i+1}\big]\right)
    \E_{|b|}\big[|b|^{\alpha_{k+1}}\big]\;.\label{eq:11}
\end{align}
Since $\E[x_i^{\alpha_i+1}]=0$ if $\alpha_i$ is even,
for a term in~\eqref{eq:11} to be non-zero it is necessary
that $\alpha_i$ is odd for every $1\le i\le k$. Consequently,
since $\sum_{i=1}^{k+1}\alpha_i=\nu$, in any non-zero term
the parity of $\alpha_{k+1}$ is equal to the parity of
$\nu-k$. Therefore, every non-zero term is of the form
\begin{align}
    \E_{|w|_T,|b|,x_T}\left(\prod_{i=1}^kx_i\right)
    \left(\prod_{i\in I}z_i\right)
    &=\left(\prod_{i=1}^k\tau_i\right)\sign(b)^{\mathds{1}[\nu-k\text{ odd}]}
    \cdot\left(\prod_{i=1}^k\E_{|w|_i}\big[|w_i|^{\alpha_i}\big]\right)
    \E_{|b|}\big[|b|^{\alpha_{k+1}}\big]\\
    &=(-1)^{C'(\tau_T,\sign(b))}\cdot\left(\prod_{i=1}^k\E_{|w|_i}\big[|w_i|^{\alpha_i}\big]\right)
    \E_{|b|}\big[|b|^{\alpha_{k+1}}\big]\;.\label{eq:12}
\end{align}
We now establish the first case from~\eqref{eq:09}. If
$\nu<k$, then since $\nu=\sum_{i=1}^{k+1}\alpha_i$ at least one
of $\alpha_i$, $1\le i\le k$ must be zero, and therefore even.
Consequently, each term in~\eqref{eq:10} is zero and it follows
that $\E_{|w|_T,|b|,x_T}M_T(x)G^\nu=0$.

On the other hand, for $\nu\ge k$, there exists a non-zero term,
for example taking $\alpha_1=\ldots=\alpha_k=1$ and
$\alpha_{k+1}=\nu-k$. Take any such term arising from
$I\in[k+1]^\nu$. Since
$w_i,b\sim\mathcal{N}(0,1/n)$, we have
$\E_{|w|_i}\big[|w_i|^j\big],E_{|b|}\big[|b|^j\big]=C_j\cdot n^{-j/2}$
for some $C_j>0$ for every fixed $j$. Substituting in~\eqref{eq:12}
and using $\nu=\sum_{i=1}^{k+1}\alpha_i$, we get
\begin{align}
\E_{|w|_T,|b|,x_T}\left(\prod_{i=1}^kx_i\right)
\left(\prod_{i\in I}z_i\right)=
(-1)^{C'(\tau_T,\sign(b))}
C_I n^{-\nu/2}
\end{align}
for some $C_I>0$.
Therefore, $C(\nu)(-1)^{C'(\tau_T,\sign(b))}n^{-\nu/2}$
with $C(\nu)>0$ follows
since it is a sum of at most $(k+1)^\nu$ positive terms.
\end{proof}


\subsubsection{Proof of Proposition~\ref{prop:error_term}}
Let $D$ be a positive constant. We apply the decomposition
\begin{align}
  \Big|  \E_{|w|_T,|b|,x_T}
  \left[M_T(x)R_{P,n}(G) \right] \Big|
  & \leq
  \E_{|w|_T,|b|,x_T} \Big[
  \big|R_{P,n}(G)\big|\cdot \mathds{1}(|G| \leq D ) \Big]\\
  &\qquad\qquad+
  \E_{|w|_T,|b|,x_T} \Big[
  \big|R_{P,n}(G)\big|\cdot\mathds{1}(|G|>D )
  \Big]
\end{align}
The proposition follows from Lemmas~\ref{lemma:error_first_term} and~\ref{lemma:error_second_term} applied to an
arbitrary value of $D$, eg., $D=1$.

\begin{lemma}\label{lemma:error_first_term}
For any $D>0$,
\begin{align}
     \E_{|w|_T,|b|,x_T}
     \Big[ \big|R_{P,n}(G)\big|\cdot\mathds{1}(|G| \leq D )  \Big] =O\left(n^{-\frac{P+1}{2}}\right).
\end{align}
\end{lemma}

\begin{proof}

Let us observe that for a fixed $b$, $G \sim \cN(b,\frac{k}{n})$, thus
\begin{align}
    \E_{|w|_T,x_T} \left[ |R_{P,n}(G)| \mathds{1}(|G| \leq D )  \right]  = \E_{y \sim \cN(b,\frac{k}{n})} \left[ |R_{P,n}(y)| \mathds{1}(|y| \leq D )  \right].
\end{align}
Recall that $R_{P,n}(x) = \frac{\Sigma_n^{(P+1)}(\xi_x)}{(P+1)!}x^{P+1}  $ for some $|\xi_x| \leq |x|$. Thus,
\begin{align}
    \E_{y \sim \cN(b,\frac{k}{n})} \left[ |R_{P,n}(y)| \mathds{1}(|y| \leq D )  \right] \leq \sup_{|y|\leq D}  \frac{|\Sigma_n^{(P+1)}(y)|}{(P+1)!} \cdot \E_{y \sim \cN(b,\frac{k}{n}) } |y|^{P+1}.
\end{align}
On the one hand,
assuming that $n\ge 2k$,
we have $\Sigma_n=\Sigma_v$ for some $1/2\le v\le 1$,
and thus using the common polynomial bound
in property ii)
$\sup_{|y|\leq D}  |\Sigma_n^{(P+1)}(y)|\leq M_D $,
where the constant $M_D$ does not depend on $n$.
On the other hand,
\begin{align}
    \E_{y \sim \cN(b,\frac{k}{n}) } |y|^{P+1}
   & = n^{-\frac{P+1}{2}} \cdot \E_{y} \big| \sqrt{n} \cdot y|^{P+1}\\
     & \leq n^{-\frac{P+1}{2}} \cdot  2^{P+1} \cdot  \left(  |\sqrt{n}b |^{P+1} + \E_{z \sim \cN(0,k)}|z  |^{P+1} \right)\\
     & = n^{-\frac{P+1}{2}} \cdot  2^{P+1} \cdot  \left(  |\sqrt{n}b |^{P+1} +\frac{ (2k)^{\frac{P+1}{2}}\Gamma(\frac{P+2}{2})}{\sqrt{\pi}} \right),
\end{align}
where in the last equation we plugged the (P+1)-th central moment of the Gaussian distribution (see, eg.,~\cite{Winkelbauer2012MomentsAA}). Since $|\sqrt{n}b|$ is also distributed like
an absolute value of $\mathcal{N}(0,1)$, taking the expectation over $|b|$, we get that for fixed $P,k$,
\begin{align}
   \E_{|w|_T,|b|,x_T} \Big[ \big|R_{P,n}(G)\big|\cdot\mathds{1}(|G| \leq D )  \Big] = O\left(n^{-\frac{P+1}{2}}\right).
\end{align}
\end{proof}

\begin{lemma} \label{lemma:error_second_term}
For any constant $D>0$, there exist $C_1,C_2>0$ such that
\begin{align}
     \E_{|w|_T,|b|,x_T}
     \Big[ \big|R_{P,n}(G)\big|\cdot\mathds{1}(|G| > D )  \Big] \leq C_1 \exp(-C_2 n).
\end{align}
\end{lemma}
\begin{proof}
By Cauchy-Schwartz inequality,
\begin{align}
\E_{|w|_T,|b|,x_T} \left[ |R_{P,n}(G)| \mathds{1}(|G| > D ) \right]
&\overset{(C.S.)}{\leq} \E_{|w|_T,|b|,x_T} [R_{P,n}(G)^2]^{1/2} \cdot
\Prr_{|w|_T,|b|,x_T}[|G|>D]^{1/2}\;.
\end{align}
For the first term, we use the universal polynomial bound from property ii):
\begin{align}\label{eq:07}
\Big|\E_{|w|_T,|b|,x_T} [R_{P,n}(G)^2]\Big|
&=\E_{|w|_T,|b|,x_T}\left[\left(
\frac{\sup_{|y|\le|G|}\Sigma_n^{(P+1)}(y)}{(P+1)!}|G|^{P+1}
\right)^2\right]\\
& \le\E_{|w|_T,|b|,x_T}\left[\left(
\frac{\sup_{|y|\le|G|}Cy^{2\ell}+C}{(P+1)!}|G|^{P+1}
\right)^2\right]\\
&=\E_{|w|_T,|b|,x_T}\left[\left(
\frac{CG^{2\ell}+C}{(P+1)!}|G|^{P+1}
\right)^2\right]
=O_n(1)\;,
\end{align}
using a similar reasoning as in \Cref{lemma:error_first_term}.


On the other hand, writing $G=G'+|b|$, we have
\begin{align}
    \Prr_{|w|_T,|b|,x_T}[|G|>D]
    &\le\Prr_{|b|}[|b|>D/2]+\Prr_{|w|_T,x_T}[|G'|>D/2]\\
    &\le 2\Prr_{y\sim\mathcal{N}(0,1/n)}[|y|>D/2]
    \le 4\exp(-D^2n/8)\;.\label{eq:08}
\end{align}
We get desired bound putting together~\eqref{eq:07} and~\eqref{eq:08}.
\end{proof}

\section{Expressivity of Common Activation Functions} \label{app:common_act}

In this section we show that ReLU and sign are expressive.
It is clear that both of these functions are polynomially
bounded, so we only need to analyze their Hermite
expansions for condition b) in \Cref{def:expressive_ab}.
In both cases we do it by writing a closed form for $\Sigma^{(k)}(0)$.

\begin{proposition}
$\ReLU(x) := \max \{0,x\}$ is expressive.
\end{proposition}
\begin{proof}
We will see that in the case $\sigma=\ReLU$ we have
$\Sigma(z)=\frac{z}{2}+\frac{z}{2}\erf(z)+\frac{1}{2\sqrt{\pi}}\exp(-z^2)$.
Indeed,
\begin{align}
    \Sigma(z)
    &=\int_{-\infty}^\infty \mathds{1}(z+y\ge 0)(z+y)\phi(y)\,\mathrm{d}y
    =\int_{-z}^\infty (z+y)\phi(y)\,\mathrm{d}y
    =z\Phi(z)+\phi(z)\\
    & =\frac{z}{2}+\frac{z\erf(z/\sqrt{2})}{2}+\phi(z)\;.
\end{align}
Using well-known Taylor expansions of $\erf$ and $\phi$, this results in
\begin{align}
    \Sigma^{(k)}(0)=\begin{cases}
    \frac{1}{2}&\text{if }k=1\;,\\
    \frac{(-1)^{k/2+1}}{\sqrt{2\pi}2^{k/2}(k-1)(k/2)!}
    &\text{if $k$ is even},\\
    0&\text{otherwise.}
    \end{cases}
\end{align}
In particular, $\Sigma^{(k)}(0)\ne 0$ for every even $k$ and
ReLU is expressive.
\end{proof}

\begin{proposition}
The sign function $\sign(x)$ is expressive.
\end{proposition}

\begin{proof}
In this case, similarly, we have
\begin{align}
    \Sigma(z)&=-\int_{-\infty}^{-z}\phi(z)\,\mathrm{d}z
    +\int_{-z}^\infty\phi(z)\,\mathrm{d}z
    =2\Phi(z)-1=\erf(z/\sqrt{2})\;,
\end{align}
which can be seen to have the expansion
\begin{align}
    \Sigma^{(k)}(0)=\begin{cases}
    \frac{2}{\sqrt{\pi}}\cdot
    \frac{(-1)^{(k-1)/2}}
    {2^{k/2}\left(\frac{k-1}{2}\right)!k}
    &\text{if $k$ is odd,}\\
    0&\text{otherwise.}
    \end{cases}
\end{align}
Again, the sign function is expressive
since $\Sigma^{(k)}\ne 0$ for every odd $k$.
\end{proof}

\section{Proof of Proposition~\ref{prop:general_functions}}
Using the definition of INAL and the Fourier expansion of $f$, we get
\begin{align}
    \INAL(f,\sigma) & =   \E_{w,b} \left[ \langle   f,\sigma \rangle^2 \right]\\
    & = \E_{w,b} \left[ \left(\sum_{T\in [n]} \hat f(T) \langle M_T, \sigma \rangle  \right)^2 \right]\\
    & = \E_{w,b} \left[ \sum_{T} \hat f(T)^2 \langle M_T, \sigma \rangle^2  + \sum_{S\neq T} \hat f(S) \hat f(T)\langle M_T, \sigma \rangle \langle M_S, \sigma \rangle \right]. \label{eq:cross_term}
\end{align}
We show that the second term of~\eqref{eq:cross_term} is zero. Let $S,T$ be two distinct sets. Without loss of generality, assume that $|S|\geq |T|$, and let $i $ be such that $i \in S$ but $i \not\in T$ (such $i$ must exist since $S \neq T$).
Fix $w$ and $b$ and decompose $w$ into $w_{\sim i},|w_i|,\sign(w_i)$
where $w_{\sim i} $ denotes the vector of weights, excluding coordinate $i$.
By applying the change of variable $\sign(w_i) x_i \mapsto y_i$ and noticing that $x_i$ has the same distribution of $y_i$, we then get
\begin{align}
    \langle M_S, \sigma \rangle& = \E_{x}[M_S(x) \cdot \sigma(x_i \sign(w_i) |w_i| + \sum_{j \neq i} x_j w_j  +b) ] \\
    & =  \sign(w_i) \cdot \E_{x_{\sim i},y_i}[M_{S_{\sim i}}(x)\cdot y_i \cdot  \sigma(y_i|w_i| + \sum_{j\neq i} x_j w_j +b) ]\\
    & := \sign(w_i) \cdot E_S,
\end{align}
where $E_S$ does not depend on $\sign(w_i)$.
On the other hand,
\begin{align}
    \langle M_T, \sigma \rangle & = \E_{x}[M_T(x) \cdot  \sigma(x_i \sign(w_i) |w_i| + \sum_{j \neq i} x_j w_j  +b) ] \\
    & = \E_{x_{\sim i},y_i}[M_{T}(x)\cdot  \sigma(y_i|w_i| + \sum_{j\neq i} x_j w_j +b) ],
\end{align}
which means that $\langle M_T, \sigma \rangle$ does not depend on $\sign(w_i)$.
Thus, we get
\begin{align}
    \E_{w,b}\left[\langle M_T,\sigma\rangle\langle M_S,\sigma\rangle\right]
    =\E_{w,b}\left[ \sign(w_i) \cdot \langle M_T, \sigma \rangle \cdot E_S \right] =0.
\end{align}
Hence,
\begin{align}
    \INAL(f,\sigma) & = \sum_T \hat f(T)^2 \E_{w,b} \left[ \langle M_T, \sigma \rangle^2\right]\\
    & = \sum_T \hat f(T)^2\INAL(M_T,\sigma).
\end{align}

\section{Proof of Corollary~\ref{cor:INAL_small_high_degree}}
Indeed, by \Cref{prop:general_functions} for any $f:\{\pm 1\}^n\to\mathbb{R}$
and $k$ it holds
\begin{align}
    \INAL(f,\sigma) &= \sum_T \hat f(T)^2 \INAL(M_T,\sigma)
    \ge W^k(f)\INAL(M_k,\sigma)\;.
\end{align}
Accordingly, if $\INAL(M_k,\sigma)=\Omega(n^{-k_0})$, we have
\begin{align}
    W^{k}(f_n)\le\INAL(f_n, \sigma)\cdot O(n^{k_0})\;,
\end{align}
and then, under our assumptions, also
\begin{align}\label{eq:01}
    W^{\le k}(f_n)\le\INAL(f_n,\sigma)\cdot O(n^{k_0})\;.
\end{align}
For the ``in particular'' statement, let $(f_n)$ be a function family
with negligible $\INAL(f_n,\sigma)$ for a correlating $\sigma$.
Let $k\in\mathbb{N}$. Since $\sigma$ is correlating,
the assumption $\INAL(M_{k'},\sigma)=\Omega(n^{-k_0})$ for
$k'=0,\ldots,k$ holds. Therefore,~\eqref{eq:01} also holds
and $W^{\le k}(f)$ is negligible. Since $k$ was arbitrary,
the function family $(f_n)$ is high-degree.

\section{Details and Proof of Corollary \ref{cor:learning}}
\label{app:fully-connected}

\cref{cor:learning} states a hardness results for learning on
\emph{fully connected neural networks with iid initialization}.
This is a more specific definition than the one we gave for
a neural network in \cref{sec:definitions}. Let us
state it precisely, following the treatment
in~\cite{abbe2020poly}.

\begin{definition}
\label{def:nn-detailed}
For the purposes of \cref{cor:learning}, a neural network on
$n$ inputs consists
of a differentiable activation function $\sigma:\mathbb{R}\to\mathbb{R}$,
a threshold function $f:\mathbb{R}\to\{\pm 1\}$ and a weighted,
directed graph with a vertex set labeled with
$\{1,x_1,\ldots,x_n,v_1,\ldots,v_m,v_{\mathrm{out}}\}$.
The vertices labeled with $x_1,\ldots,x_n$ are called the input
vertices, the vertex labeled with $1$ is the constant vertex
and $v_{\mathrm{out}}$ is the output vertex.

We assume that the graph does not contain loops,
the constant and input vertices do not have any incoming
edges, the output vertex does not have outgoing edges
and for the remaining vertices there are no edges
$(v_i,v_j)$ for $i>j$. Each vertex (a neuron) has an associated
function (the output of the neuron)
from $\mathbb{R}^n$ to $\mathbb{R}$ which
is defined recursively as follows:
The output of the constant vertex is $y_1 = 1$ and the output of the
input vertex is (abusing notation) $y_{x_i} = x_i$.
The output of any other vertex $v_i$ is given by
$y_{v_i}=\sigma(\sum_{v:(v,v_i)\in E(G)} w_{v,v_i} y_v)$.
Finally, the output of the whole network is given by
$f(y_{v_{\mathrm{out}}})$.

We say that the neural network is \emph{fully connected}
if every vertex that has an incoming edge from an input vertex
has incoming edges from all input vertices.
\end{definition}

Note that our definition of ``fully connected network''
covers any feed-forward
architecture that consists of a number of fully connected
hidden layers stacked on top of each other.

Let us restate Theorem~3 from~\cite{abbe2020poly}
with the
bound\footnote{Since we are discussing GD, we are applying
their bound with infinite sample size
$m=\infty$.}
from their Corrolary~1 applied to
the junk flow term JF$_T$:
\begin{theorem}[\cite{abbe2020poly}]
\label{thm:abbe-sandon}
Let $P_{\mathcal{F}}$ be a distribution on Boolean functions
$f:\{\pm 1\}^n\to\{\pm 1\}$.
Consider any neural network
as defined in \cref{def:nn-detailed} with $E$ edges.
Assume that a function $f$ is chosen from $P_\mathcal{F}$
and then $T$ steps of noisy GD
with learning rate $\gamma$,
overflow range $A$ and noise level $\tau$ are run
on the initial network using function $f$ and uniform
input distribution $\mathcal{U}^n$.

Then, in expectation over the initial choice of $f$,
the training noise, and a fresh sample $x\sim\mathcal{U}^n$,
the trained neural network $\NN^{(T)}$ satisfies
\begin{align}
    \Pr
    \big[\NN^{(T)}(x) = f(x)\big]
    &\le \frac{1}{2}+\frac{\gamma T\sqrt{E}A}{\tau}
    \cdot \CP(P_{\mathcal{F}},\mathcal{U}^n)^{1/4}\;.
\end{align}
\end{theorem}

Finally, we need to discuss the fact that \cref{cor:learning}
applies for any fully connected neural network
\emph{with iid initialization}. What we mean by this
is that the initial neural network has a fixed activation
$\sigma$, threshold function $f$ and graph (vertices and edges),
but the weights on edges are not fixed. Instead, they are chosen
randomly iid from any fixed probability distribution.
More precisely, we can make a weaker assumption
that the weights on all edges that are outgoing
from the input vertices are
chosen\footnote{
Even more precisely, we can assume only that the distribution
of these weights is symmetric under permutations of
input vertices $x_1,\ldots,x_n$.}
iid from a fixed distribution and all the other weights
have arbitrary fixed values.

We can now proceed to prove \cref{cor:learning}.

\subsection{Proof of Corollary \ref{cor:learning}}

Let a randomly initialized, fully connected neural network
$\NN$ be trained in the following
way. First, a function $\overline{f_n}\circ \pi$ is chosen
uniformly at random from the orbit of $\overline{f_n}$.
Then, a noisy GD algorithm is run with the parameters stated:
$T$ steps, learning rate $\gamma$, overflow range $A$ and
noise level $\tau$.
Finally, a fresh sample $x\sim\mathcal{U}^N$ is presented
to the trained neural network. Then, \cref{thm:abbe-sandon}
says that
\begin{align}
    \Pr\big[\NN^{(T)}(x)=(\overline{f_n}\circ\pi)(x)\big]
    \le\frac{1}{2}+\frac{\gamma T\sqrt{E}A}{\tau}
    \cdot\CP(\orb(\overline{f_n}),\mathcal{U}^N)^{1/4}\;.
\end{align}
Since we can apply \cref{thm:abbe-sandon} to the class
of all orbits of $-\overline{f_n}$, which has the same
cross-predictability, the same upper bound also holds
for $\Pr[\NN^{(T)}(x)\ne (\overline{f_n}\circ\pi)(x)]$.
Consequently, we have the expectation bound
\begin{align}\label{eq:19}
    \Big|\E\langle \NN^{(T)}, \overline{f_n}\circ\pi\rangle\Big|
    \le\frac{2\gamma T\sqrt{E}A}{\tau}
    \cdot\CP(\orb(\overline{f_n}),\mathcal{U}^N)^{1/4}\;.
\end{align}

Recall that the neural network
is fully connected and the weights on the edges outgoing
from the input vertices are iid. The expectation
in~\eqref{eq:19} is an average of conditional expectations
for different initial choices of permutation $\pi$.
Consider the action induced by $\pi$ on the weights outgoing
from the input vertices. By properties of GD, it follows
that each conditional expectation over $\pi$ contributes
equally to the left-hand side of~\eqref{eq:19}. It follows
that the same bound holds also for the single function
$\overline{f_n}$:
\begin{align}\label{eq:05}
    \left|\E_{\NN^{(T)}}\langle\NN^{(T)},\overline{f_n}\rangle\right|
    \le\frac{2\gamma T\sqrt{E}A}{\tau}\CP(\orb(\overline{f_n}),\cU^N)^{1/4}\;.
\end{align}

Accordingly, if $\INAL(f_n,\sigma)$ is negligible, then, by
\Cref{thm:mainthm}, $\CPMacro$ is negligible and the right-hand side of~\eqref{eq:05}
remains negligible for any polynomial bounds on
$\gamma$, $T$, $E$, $A$ and $\tau$, as claimed.

For the more precise statement,
if $\INAL(f_n, \sigma)=O(n^{-c})$,
then again by \Cref{thm:mainthm} it holds
$\CPMacro=O(n^{-\frac{\epsilon}{1+\epsilon}(c-1)})$ and
we get the bound of
$O\left(\frac{\gamma T\sqrt{E}A}{\tau}\cdot
n^{-\frac{\epsilon}{4(1+\epsilon)}(c-1)}
\right)$ on the right-hand side of~\eqref{eq:05}.
\hfill\qedsymbol

\end{document}